\documentclass[conference]{IEEEtran}
\usepackage{cite,times}
\usepackage{amsfonts}
\usepackage{amsmath}
\usepackage{float}
\usepackage{array}
\usepackage{caption}
\usepackage{subcaption}
\usepackage{graphics}
\usepackage{graphicx, color}
\usepackage{epsf}
\usepackage{amssymb}
\usepackage{stackrel}
\usepackage{algorithm}
\usepackage{algorithmic}
\usepackage{dblfloatfix}
\usepackage{fixltx2e}
\usepackage{multicol}
\usepackage{blindtext, subfig}
\usepackage{soul}
\usepackage{dblfloatfix} 

\newtheorem{propo}{Proposition}[section]
\newtheorem{lemma}[propo]{Lemma}

\newtheorem{coro}[propo]{Corollary}
\newtheorem{thm}[propo]{Theorem}

\newtheorem{remark}[propo]{Remark}

\usepackage{color}

\def\reals{\mathbb{R}}

\def \sT{{\sf T}}

\def\<{\langle}
\def\>{\rangle}
\def\de {{\rm d}}
\def\ind{\mathbb{I}}

\begin{document}

\title{1-Bit Matrix Completion under Exact Low-Rank Constraint \vspace*{-0.3in}}

\author{\IEEEauthorblockN{\textit{Sonia A.\ Bhaskar} and \textit{Adel Javanmard} \vspace*{-0.15in}} \\
\IEEEauthorblockA{Department of Electrical Engineering, Stanford University, Stanford, CA 94304, USA\\
 Email: sbhaskar@stanford.edu, adelj@stanford.edu}}

\maketitle
\thispagestyle{empty}
\pagestyle{empty}

\begin{abstract}
We consider the problem of noisy 1-bit matrix completion under an exact rank constraint on the true underlying matrix $M^*$. Instead of observing a subset of the noisy continuous-valued entries of a matrix $M^*$, we observe a subset of noisy 1-bit (or binary) measurements generated according to a probabilistic model.  We consider constrained maximum likelihood estimation of $M^*$, under a constraint on the entry-wise infinity-norm of $M^*$ and an exact rank constraint. This is in contrast to previous work which has used convex relaxations for the rank. We provide an upper bound on the matrix estimation error under this model. Compared to the existing results, our bound has faster convergence rate with matrix dimensions when the fraction of revealed 1-bit observations is fixed, independent of the matrix dimensions. We also propose an iterative algorithm for solving our nonconvex optimization with a certificate of global optimality of the limiting point.  This algorithm is based on low rank factorization of $M^*$.  We validate the method on synthetic and real data with improved performance over existing methods. 
\end{abstract}

\section{INTRODUCTION}
The problem of recovering a low rank matrix from an incomplete or noisy sampling of its entries arises in a variety of applications, including collaborative filtering \cite{Koren09} and sensor network localization \cite{Shang04, Karbasi13}.   In many applications, the observations are not only missing, but are also highly discretized, e.g.\ binary-valued (1-bit) \cite{Davenport12, Cai13}, or multiple-valued \cite{Lan14a}. For example, in the Netflix problem where a subset of the users' ratings is observed, the ratings take integer values between 1 and 5.  Although one can apply existing matrix completion techniques to discrete-valued observations by treating them as continuous-valued, performance can be improved by treating the values as discrete \cite{Davenport12}. 

In this paper we consider the problem of completing a matrix from a subset of its entries, where instead of observing continuous-valued entries, we observe a subset of 1-bit measurements.
Given $M^* \in \reals^{m \times n}$, a subset of indices $\Omega \subseteq [m] \times [n]$, and a twice differentiable function $f\,:\, \reals \rightarrow [0,1]$, we observe (``w.p.''  stands for ``with probability'')
\begin{equation} \label{e1.1a}
   Y_{ij} = \begin{cases}
          +1 & \mbox{w.p.} \quad f(M^*_{ij})\,, \\
          -1 & \mbox{w.p.} \quad 1-f(M^*_{ij}) 
          \end{cases} 
          \quad \text{ for } (i,j)\in \Omega .
\end{equation}
One important application is the binary quantization of $Y_{ij} = M^*_{ij} + Z_{ij}$, where $Z$ is a noise matrix with i.i.d entries. If we take $f$ to be the cumulative distribution function of $-Z_{11}$, then the model in~\eqref{e1.1a} is equivalent to observing 
\begin{equation}\label{e1.1}
Y_{ij} = \begin{cases}
+1 & \mbox{if } \quad  M^*_{ij}+Z_{ij} > 0 \\
-1 & \mbox{if }  \quad  M^*_{ij}+Z_{ij} < 0
\end{cases}
\quad\; \text{  for } (i,j)\in \Omega.
\end{equation}
Recent work in the 1-bit matrix completion literature has followed the probabilistic model in \eqref{e1.1a}-\eqref{e1.1} for the observed matrix $Y$ and has estimated $M^*$ via solving a constrained maximum likelihood (ML) optimization problem. Under the assumption that $M^*$ is low-rank, these works have used convex relaxations for the rank via the trace norm \cite{Davenport12} or max-norm \cite{Cai13}. An upper bound on the matrix estimation error is given under the assumptions that the entries are sampled according to a uniform distribution \cite{Davenport12}, or in \cite{Cai13}, following a non-uniform distribution. 

In this paper, we follow \cite{Davenport12, Cai13} in seeking an ML estimate of $M^*$ but use an exact rank constraint on $M^*$ rather than a convex relaxation for the rank.  We follow the sampling model of \cite{Bhojanapalli:2014kx} for $\Omega$ which includes the uniform sampling of \cite{Davenport12} as well as non-uniform sampling.  We provide an upperbound on the Frobenius norm of matrix estimation error, and show that our bound yields faster convergence rate with matrix dimensions than the existing results of \cite{Davenport12, Cai13} when the fraction of revealed 1-bit observations is fixed independent of the matrix dimensions.  Lastly, we present an iterative algorithm for solving our nonconvex optimization problem with a certificate of global optimality under mild conditions. Our algorithm outperforms \cite{Davenport12, Cai13} in the presented simulation example.

{\bf Notation}: 
For matrix $A$ with $(i,j)$-th entry $A_{ij}$,  we use the notation $\|A\|_\infty = \underset{i,j}{\max} |A_{ij}|$ for the entry-wise infinity-norm, $\|A\|_F$ for the Frobenius norm and $\|A\|_2$ for its operator norm. We use $A_{i,\cdot}$ to denote the $i$-th row and $A_{\cdot,j}$ to denote the $j$-th column. Taking $\mathcal{S}$ to be a set, we use $|\mathcal{S}|$ to denote the cardinality of $\mathcal{S}$. The notation $[n]$ represents the set of integers $\{1,\hdots,n\}$.  We denote by $\mathbf{1}_n \in \reals^{n}$ the vector of all ones, by $\tilde{\mathbf{1}}_n$ the unit vector $\mathbf{1}_n / \sqrt{n}$, and by $\ind_{\mu}$ the indicator function, i.e. $\ind_{\mu} = 1$ when $\mu$ is true, else $\ind_{\mu} = 0$.
%
\section{MODEL ASSUMPTIONS} \label{sec:model}
We wish to estimate unknown $M^*$ using a constrained ML approach. We use $M \in \reals^{m \times n}$ to denote the optimization variable. Then the negative log-likelihood function for the given problem is 
\begin{align} 
   F_{\Omega, Y}(M) =  - \sum_{(i,j) \in \Omega} &\Big\{\ind_{(Y_{ij}=1)} \log(f(M_{ij}))  \nonumber \\
     & + \ind_{(Y_{ij}=-1)} \log(1-f(M_{ij})) \Big\} \label{e1.2}   
\end{align}
Note that \eqref{e1.2} is a convex function of $X$ when the function $f$ is log-concave.  Two common choices for which the function $f$ is log-concave are: $(i)$ Logit model with logistic function $f(x)= 1/(1+e^{-x/\sigma})$ and parameter $\sigma >0$, or equivalently $Z_{ij}$ in (\ref{e1.1}) is logistic with scale parameter $\sigma$; $(ii)$ Probit model with $f(x) = \Phi(x/\sigma)$
where $\sigma >0$ and $\Phi(x)$ is the cumulative distribution function of  ${\cal N}(0,1)$. 
We assume that $M^*$ is a low-rank matrix with rank bounded by $r$, and that the true matrix $M^*$ satisfies $\|M^*\|_\infty \le \alpha$, which helps make the recovery of $M^*$ well-posed by preventing excessive ``spikiness" of the matrix.  We refer the reader to \cite{Davenport12, Cai13} for further details.

The constrained ML estimate of interest is the solution to the optimization problem (s.t.: subject to):
\begin{equation} \label{e1.3}
   \widehat{M} = \arg\min_{M} F_{\Omega, Y}(M) \,\;\, \mbox{s.t.} \, \|M\|_\infty \le \alpha , \, {\rm rank}(M) \le r .
\end{equation}
In many applications, such as sensor network localization, collaborative filtering, or DNA haplotype assembly, the rank $r$ is known or can be reliably estimated \cite{keshavan2010matrix}. 

We now discuss our assumptions on the set $\Omega$. Consider a bipartite graph $G = ([m],[n], E)$, where the edge set $E \subseteq [m] \times [n]$ is related to the index set of revealed entries $\Omega$  as $(i,j) \in E$ iff $(i,j) \in \Omega$. Abusing the notation, we use $G$ for both the graph and its bi-adjacency matrix where $G_{ij}=1$ if $(i,j) \in E$, $G_{ij}=0$ if $(i,j) \not\in E$. 
We denote the association of ${G}$ to ${\Omega}$ by ${G} \backslash {\Omega}$. Without loss of generality we take $m \ge n$.
We assume that each row of $G$ has $d$ nonzero entries (thus $|\Omega| = m d $) with the following properties on its SVD:
\begin{itemize}
\item[{\bf (A1)}] The left and right top singular vectors of ${G}$ are $\mathbf{1}_{m} / \sqrt{m}$ and $\mathbf{1}_{n} / \sqrt{n}$, respectively.
This implies that $\sigma_1(G) = d\sqrt{m/n} \ge d$, where $\sigma_1(G)$ denotes the largest singular value of $G$, and that each column of $G$ has $(md/n)$ nonzero entries.
\item[{\bf (A2)}] We have $\sigma_2({G}) \leq C\sqrt{d}$, where $\sigma_2(G)$ denotes the second largest singular value of $G$ and $C>0$ is some universal constant. 
\end{itemize}
Thus we require ${G}$ to have a large enough spectral gap. As discussed in \cite{Bhojanapalli:2014kx}, an Erd\"os-Renyi random graph with average degree $d \ge c  \log(m)$ satisfies this spectral gap property with high probability, and so do stochastic block models for certain choices of inter- and intra-cluster edge connection probabilities. Thus, this sampling scheme is more general than a uniform sampling assumption, used in \cite{Davenport12}, and it also includes the stochastic block model \cite{Bhojanapalli:2014kx} resulting in non-uniform sampling. 
\vspace{-0.1in}
\section{PERFORMANCE UPPERBOUND} \label{sec:main}
We now present a performance bound for the solution to \eqref{e1.3}. With $\dot{f} (x) := (\de f(x)/\de x)$, define
\begin{align}
   \gamma_{\alpha} \le & \min \left(  \inf_{|x| \le \alpha} \left\{ \frac{ \dot{f}^2(x)}{f^2(x)} - \frac{ \ddot{f}(x)}{f(x)} \right\} 
       \right. , \nonumber \\
      & \hspace*{0.2in} \left. \inf_{|x| \le \alpha} \left\{ \frac{ \dot{f}^2(x)}{(1-f(x))^2} + \frac{ \ddot{f}(x)}{1-f(x)} \right\} \right) , \label{geq104} 
\end{align}
\begin{equation}
   L_{\alpha} \ge  \sup_{|x| \le \alpha} \left\{ \frac{ \left| \dot{f}(x) \right|}{f(x)(1-f(x))}  \right\}\,,\label{geq108} 
\end{equation}
where $\alpha$ is the bound on the entry-wise infinity-norm of $\widehat{M}$ (see \eqref{e1.3}). For the logit model, we have $L_{ \alpha } = 1/\sigma$, and $\gamma_{\alpha} = \frac{e^{\alpha/\sigma}}{\sigma^2 (1+e^{\alpha/\sigma})^2} \approx e^{-\alpha/\sigma} > 0$. For the probit model we obtain $L_{ \alpha } \le \frac{4}{\sigma} \left(\frac{\alpha}{\sigma} +1 \right)$, $\gamma_{\alpha}  \ge \frac{\alpha}{\sqrt{2 \pi} \sigma^3} \exp \left(-\frac{\alpha^2}{2 \sigma^2} \right) > 0$.
For further reference, define the constraint set 
\begin{equation} \label{e1.4}
   {\mathcal C} := \left\{ M \in \reals^{m \times n} \,:\,  \|M\|_\infty \le \alpha, \; {\rm rank}(M) \le r   \right\}\,.
\end{equation}

\begin{thm} \label{thm:main_gd} Suppose that $M^* \in {\mathcal C}$, 
and $G\backslash\Omega$ satisfies assumptions (A1) and (A2), with $m \geq n$. Further, suppose $Y$ is generated according to (\ref{e1.1a})
and $f(x)$ is log-concave in $x$. Then with probability at least 
$1 -  C_1 \exp(-C_2 m) $,
any global minimizer $\widehat{M}$ of (\ref{e1.3}) satisfies
\begin{align}  
     \frac{1}{\sqrt{mn}} &  \| \widehat{M} - M^* \|_F \le
       \max \left( \frac{ C_{1 \alpha} r \sigma_2(G)}{\sigma_1(G)} ,  
                \frac{C_{2 \alpha} m \sqrt{r^3 n}}{\sigma_1^2(G)}  \right) \label{thm_eq1} \\
    & \le  \max \left( \frac{ C_{1 \alpha} C r \sqrt{m}}{\sqrt{|\Omega|}},  
                \frac{C_{2 \alpha} m^3 \sqrt{r^3 n}}{|\Omega|^2}  \right)\,, \label{thm_eq2}
\end{align}
provided $\gamma_{\alpha} >0 $. Here, $C_1, C_2 > 0$ are universal constants, $C > 0$ is given by assumption (A2), 
and 
\[C_{1\alpha} \equiv 4\sqrt{2}\alpha\,, \quad C_{2\alpha} \equiv 32.16\sqrt{2} L_{\alpha}/\gamma_\alpha\,,\]
with $\gamma_{\alpha}$ and $L_{2 \alpha}$ given by (\ref{geq104}), (\ref{geq108}).
\end{thm}
Proof of this theorem is given in Sec.\ \ref{sec:proof}.
Of particular interest is the case where $p=\frac{|\Omega|}{mn}$ is fixed and we let $m$ and $n$ become large, with $m/n \equiv \delta \ge 1$ fixed. In this case we have the following Corollary. 
\begin{coro} \label{cor:main_gd} Assume the conditions of Theorem \ref{thm:main_gd}. Let $p=\frac{|\Omega|}{m n}$ be fixed independent of $m$ and $n$.  Then with probability at least 
$1 -  C_1 \exp (-C_2 m) $,
any global minimum $\widehat{M}$ to (\ref{e1.3}) satisfies
\begin{align}  
     \frac{1}{\sqrt{mn}} &  \| \widehat{M} - M^* \|_F  
            \le \mathcal{O} \left( \frac{\delta}{p^2} \sqrt{\frac{r^3}{n}} \right) .
                               \label{geq232}
\end{align}
\end{coro}

\subsection{Comparison with previous work} \label{sec:comp}
Consider $M^* \in \reals^{n\times n}$, with $p$ fraction of its entries
sampled, such that $\|M^*\|_\infty \leq \alpha$ (also assumed in \cite{Davenport12,Cai13}) and $\text{rank}(M^*) \leq r$. Then 
$m =n$, and $| \Omega |= p n^2$. 
The bounds proposed in \cite{Davenport12} (and \cite{Cai13} in case of uniform sampling) yields
\begin{equation} \label{geq81}
  \frac{1}{n^2} \|\widehat{M} - M^*\|_F^2 \le  \mathcal{O} \left( \sqrt{\frac{r}{pn}} \right)\,,
\end{equation}
whereas, applying our result (Corollary 3.2), we obtain
\begin{align} 
  \frac{1}{n^2} \|\widehat{M} - M^*\|_F^2 & 
              \le \mathcal{O} \left(  \frac{r^3 }{p^4 n}  \right)\,.  \label{geq81b}
\end{align} 
Comparing (\ref{geq81}) and (\ref{geq81b}), we see our method has faster convergence rate in $n$ for fixed rank $r$ and fraction of revealed entries $p$. Notice that if the number of missing entries scales with $n$ according to $ p \sim \Theta(1/n)$, \cite{Davenport12} yields bounded error while our bound grows with $n$; in our case we need $ p $ to be of order at least $n^{-1/4}$.  We believe this to be an artifact of our proof, as our numerical results (Fig.\ \ref{fig:sim2}) show our method outperforms \cite{Davenport12}, especially for low values of $p$ and higher values of rank $r$.

\section{OPTIMIZATION} \label{sec:alg}
We will solve the optimization problem \eqref{e1.3} using a log-barrier penalty function approach \cite[Sec.\ 11.2]{Boyd2004}.  The constraint $\max_{i,j}|M_{ij}| \le \alpha$ translates to the log-barrier penalty function $- \log \left( 1- (M_{ij}/\alpha)^2 \right)$.
This leads to the regularized objective function
\begin{equation} \label{e4.2}
   \overline{F}_{\Omega, Y}(M) = F_{\Omega, Y}(M) - \lambda \sum_{(i,j)} \log \left( 1- (M_{ij}/\alpha)^2 \right) 
\end{equation}
\vspace{-0.01in}
and the optimization problem
\begin{equation} \label{e4.3}
   \widehat{M} = \arg\min_{M} \overline{F}_{\Omega, Y}(M) \; \mbox{subject to} \;  {\rm rank}(M) \le r .
\end{equation}
\vspace{-0.01in}
We can account for the rank constraint in \eqref{e4.3} via the factorization technique of \cite{Bach08, Burer03, Recht13} where instead of optimizing with respect to $M$ in \eqref{e1.3}, $M$ is factorized into two matrices $U \in \mathbb{R}^{m \times k}$ and $V \in \mathbb{R}^{n \times k}$ such that $M = UV^\top$.  One then chooses $k = r+1$ and optimizes  with respect to the factors $U,V$. The reformulated objective function is then given by 
\begin{align} 
 \check{F}_{\Omega, Y}(U,V)  = F_{\Omega, Y}(UV^\top) 
       - \lambda \sum_{(i,j)} 
         \log \left( 1- (U_{i,\cdot}V_{j,\cdot}^\sT / \alpha )^2 \right) \label{e4.4}
\end{align}
where $U_{i,\cdot}$ denotes the $i\mbox{-}$th row of $U$, and $V_{j,\cdot}$ the $j\mbox{-}$th row of $V$. The parameter $\lambda > 0$ sets the accuracy of approximation of $\max_{i,j}|M_{ij}| \le \alpha$ via the log-barrier function. We solve this factored version using a gradient descent method with backtracking line search, in a sequence of central path following solutions \cite[Sec.\ 11.2]{Boyd2004}, where one gradually reduces $\lambda$ toward 0. Initial values of $U,V$ are randomly picked and scaled to satisfy $\|UV^\top\|_\infty \le 0.95 \alpha$. Starting with a large $\lambda_0$, we solve for $\lambda = \lambda_0, \lambda_0/2, \lambda_0/4, \cdots $ via central path following and use 5-fold cross validation error over $\lambda$ as the stopping criterion in selecting $\lambda$.  
%

%
\begin{figure}[h!]
        \centering
                \hspace*{0.4in}\includegraphics[width=3.1in]{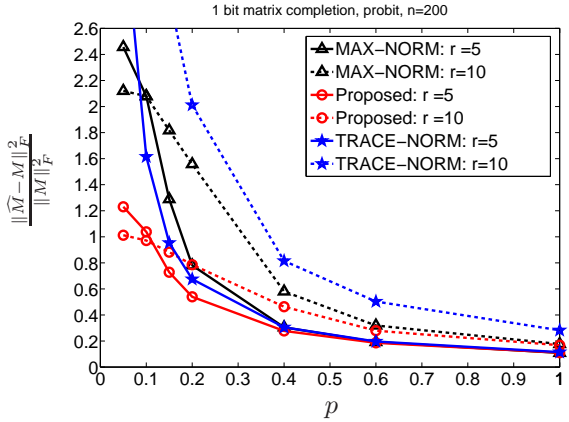}
                 \put(-110,0){$p$}
                \put(-230,70){\rotatebox{90}{{\small $\frac{\|\widehat{M}-M\|_F^2}{\|M\|_F^2}$}}}
                       \caption{ Relative MSE $\|\widehat{M}-M^*\|_F^2/ \|M^*\|_F^2$ for varied values of $p$. Probit model, rank $r$, $\sigma = 0.18$, $n=200$, $\alpha=1$. 
         ``trace-norm'' refers to \cite{Davenport12}, ``max-norm" is the method of \cite{Cai13} for known $r$.}
                \label{fig:sim2}
\end{figure}

\vspace*{-0.25in}
\begin{remark} \label{rem2} The hard rank constraint results in a nonconvex constraint set. Thus, \eqref{e1.3} and \eqref{e4.3} are nonconvex optimization problems; similarly for minimization of (\ref{e4.4}) for which the rank constraint is implicit in the factorization of $M$.  However, the following result is shown in \cite[Proposition 4]{Bach08}, based on \cite{Burer03}, for nonconvex problems of this form.  If $(U^*, V^*)$ is a local minimum of the factorized problem, then $\widetilde{M} = U^* {V^*}^\top$ is the global minimum of problem~\eqref{e4.3}, so long as $U^*$ and $V^*$ are rank-deficient. (Rank deficiency is a sufficient condition, not necessary.) This result is utilized in \cite{Recht13} and \cite{Cai13} for problems of this form.  
\end{remark}

\section{NUMERICAL EXPERIMENTS}\label{sec:simulation}
\subsection{Synthetic Data}
In this section, we test our method on synthetic data and compare it with the methods of \cite{Davenport12, Cai13}. We set $m=n$ and construct $M^*\in \reals^{n\times n}$ as $M^* = M_1 M_2^\top$ where $M_1$ and $M_2$ are $n \times r$ matrices with i.i.d.\ entries drawn from a uniform distribution on $[-0.5,0.5]$ (as in \cite{Cai13, Davenport12}). Then we scale $M^*$ to achieve $\| M^*\|_\infty = 1 = \alpha$. We pick $r=5, 10$, vary matrix sizes $n=100,200,$ or 400. We generate the set $\Omega$ of revealed indices via the Bernoulli sampling model of \cite{Davenport12} with $p$ fraction of revealed entries. We consider the probit ($\sigma = 0.18$, as in \cite{Cai13, Davenport12}) model. For Fig.~\ref{fig:sim2}, we take $n = 200$ and vary $p$. The resulting relative mean-square error (MSE) $\|\widehat{M}-M^*\|_F^2/ \|M^*\|_F^2$ averaged over 20 Monte Carlo runs is shown in Fig.~\ref{fig:sim2}. As expected, the performance improves with increasing $p$.  For comparison, we have also implemented the methods of \cite{Davenport12, Cai13}, labeled ``trace norm'', and ``max-norm" respectively. 
%
%
As we can see our proposed approach significantly outperforms \cite{Davenport12, Cai13}, especially for low values of $p$ and high values of $r$. 
\begin{figure}[h!]
\hspace{2mm}
                \hspace*{0.2in}\includegraphics[width=3in]{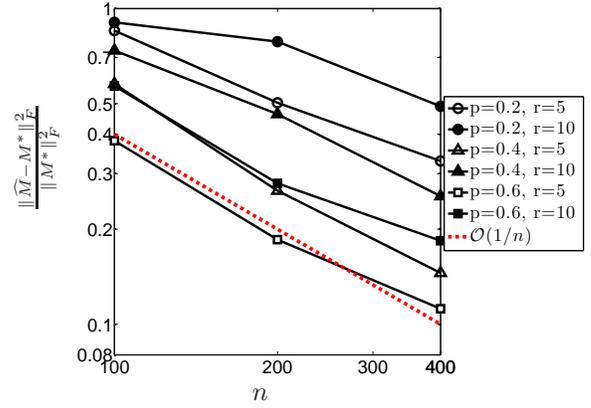}
                \put(-135,2){$n$}
                \put(-225,73){\rotatebox{90}{{\small $\frac{\|\widehat{M}-M^*\|_F^2}{\|M^*\|_F^2}$}}}
                \vspace*{-0.03in}
                \caption{ Log-log plots of relative MSE for varied $n$}
                \label{fig:sim3}
\end{figure}
\vspace{-0.25in}
\begin{figure}[h!]
\hspace{3mm}
    \includegraphics[width=0.35\textwidth]{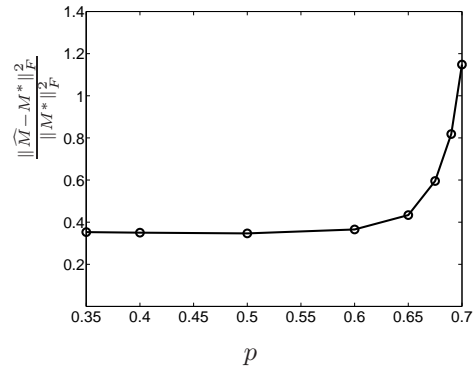}
    		\put(-100,-5){$p$}
                \put(-187,70){\rotatebox{90}{{\small $\frac{\|\widehat{M}-M^*\|_F^2}{\|M^*\|_F^2}$}}}
                \caption{Relative MSE versus $p$ under the probit model, for varied p, fixed $p+q = 0.7$}
                \label{fig:sim4}
\end{figure}
\vspace{-0.2in}

In Fig.\ \ref{fig:sim3} we show the relative MSE for $n=100,200,400$, $p=0.2, 0.4, 0.6$ for the probit model using our approach.
We also plot the line $1/n$ in Fig.\ \ref{fig:sim3} to show the scale of the upper bound $\mathcal{O} \left(  \frac{r^3 }{p^4 n}  \right)$ established in Section~\ref{sec:main}.  As we can see, the empirical estimation errors follow approximately the same scaling, suggesting that our analysis is tight, up to a constant. 

We additionally plot the MSE for $n = 200$ and $r = 5$ in Fig.\ \ref{fig:sim4}, with varying $p$ and keeping $p + q = 0.7$, under the probit model.  This enables us to study the performance of the model under nonuniform sampling.  
Note that when $p = q = 0.35$, the spectral gap is largest and MSE is the smallest, and as $p$ gets larger, the spectral gap decreases, leading to larger MSE.
\begin{table}[h!]
{\small
\begin{center}
\renewcommand{\arraystretch}{1.0}
 \begin{tabular}{||c||c|c|c||}
   \hline \hline
    & \multicolumn{3}{|c||}{\% Accuracy (Logit Model)} \\ 
    \% Training & Proposed  & Trace-norm  & Max-norm \\ \hline\hline
   95 & 72.3$\pm$0.7 & 72.4$\pm$0.6  & 71.5$\pm$0.7 \\ \hline
   10 & 60.4$\pm$0.6 & 58.5$\pm$0.5  & 58.4$\pm$0.6 \\ \hline
   5 & 53.7$\pm$0.8 & 49.2$\pm$0.7  & 50.3$\pm$0.2 \\ \hline \hline
 \end{tabular}
 \caption{Accuracy of the proposed, trace-norm \cite{Davenport12} and max-norm \cite{Cai13} approaches on the MovieLens 100k dataset for different amounts of training data.  Accuracy represents the percentage of test set ratings for which the estimate of $M^*$ accurately predicts the sign, i.e., whether the unobserved ratings were above or below the average rating.}
 \label{table1}
\end{center}
}
\end{table}

\subsection{MovieLens (100k) Dataset}
As in \cite{Davenport12}, we consider the MovieLens (100k) dataset (http://www.grouplens.org/node/73). This dataset consists of 100,000 movie ratings from 943 users on 1682 movies, with ratings on a scale from 1 to 5. Following \cite{Davenport12}, these ratings were converted to binary observations by comparing each rating to the average rating for the entire dataset.  We used three splits of the data into training/test subsets and used 20 random realizations of these splits.  The performance is evaluated by checking to see if the estimate of $M^*$ accurately predicts the sign of the test set ratings (whether the observed ratings were above or below the average rating). As in \cite{Davenport12}, we determine the needed parameter values by performing a grid search and selecting the values that lead to the best performance; we fixed $\alpha =1$, and varied $\lambda$ (i.e. central path following), $\sigma$ and rank $r$. Our performance results are shown in Table \ref{table1} using a logistic model for three approaches: proposed, \cite{Davenport12, Cai13}. These results support our findings on synthetic data that our method is preferable over \cite{Davenport12, Cai13} for sparser data.

\section{PROOF OF THEOREM \ref{thm:main_gd} } \label{sec:proof}
Our proof is based on a second-order Taylor series expansion and a matrix concentration inequality.

Let $\theta = {\rm vec} (M) \in \reals^{mn}$ and $\tilde{F}_{\Omega, Y}(\theta) = F_{\Omega, Y}(M)$. The objective function $F_{\Omega, Y}(M)$ is continuous in $M$ and the set ${\mathcal C}$ is compact, therefore, $F_{\Omega, Y}(M)$ achieves a minimum in ${\mathcal C}$. If $\widehat{\theta} = {\rm vec}(\widehat{M})$ minimizes $\tilde{F}_{\Omega, Y}(\theta)$ subject to the constraints, then $\tilde{F}_{\Omega, Y}(\widehat{\theta}) \le \tilde{F}_{\Omega, Y}(\theta^*)$ where $\theta^* = {\rm vec} (M^*)$. By the second-order Taylor's theorem, expanding around $\theta^*$ we have
\begin{align}  
    \tilde{F}_{\Omega, Y}(\theta) = & \tilde{F}_{\Omega, Y}(\theta^*) + 
             \langle \nabla_\theta \tilde{F}_{\Omega, Y}(\theta^*),\theta-\theta^* \rangle \nonumber \\
            &  + \frac{1}{2} \langle  \theta-\theta^*, \left( \nabla^2_{\theta  \theta} \tilde{F}_{\Omega, Y}(\tilde{\theta}) \right) 
              (\theta-\theta^*) \rangle  \label{geq3.0}
\end{align}
where $\tilde{\theta} = \theta^* + \gamma (\theta-\theta^*)$ for some $\gamma \in [0,1]$, with corresponding matrix $\tilde{M} = M^* + \gamma (M-M^*)$. 
We need several auxiliary results before we can prove Theorem \ref{thm:main_gd}. 

Using (\ref{e1.2}), it follows that
\begin{align} 
  \frac{\partial F_{\Omega, Y}(M)}{\partial M_{\ell k}} & 
  = \left( - \frac{\dot{f}(M_{\ell k})}{f(M_{\ell k})} \ind_{(Y_{\ell k} = 1)} \right. \nonumber \\
        & \left.  + \frac{\dot{f}(M_{\ell k})}{1-f(M_{\ell k})} \ind_{(Y_{\ell k} = -1)} \right)  \ind_{((\ell, k) \in \Omega)}  ,   
                \label{geq3.1}
\end{align}
\begin{align} 
 & \frac{\partial^2  F_{\Omega, Y}(M)}{\partial M_{\ell k}^2}  
      =  \bigg[ \left( \frac{\dot{f}^2(M_{\ell k})}{f^2(M_{\ell k})} 
            - \frac{\ddot{f}(M_{\ell k})}{f(M_{\ell k})} \right) \ind_{(Y_{\ell k} = 1)} 
            \nonumber \\
       &  + \left( \frac{\ddot{f}(M_{\ell k})}{1-f(M_{\ell k})} + \frac{\dot{f}^2(M_{\ell k})}{(1-f(M_{\ell k}))^2} \right)
                   \ind_{(Y_{\ell k} = -1)} \bigg]\,  \ind_{((\ell, k) \in \Omega)}
                 \label{geq3.2}
\end{align}
and
\begin{equation} \label{geq3.3}
  \frac{\partial^2 F_{\Omega, Y}(M)}{\partial M_{\ell_1 k_1} \partial M_{\ell_2 k_2}} = 
               0  \mbox{ if } (\ell_1, k_1) \neq (\ell_2, k_2) .
\end{equation}

Let $w \equiv {\rm vec}(M-M^*) = \theta - \theta^*$. 
Note that by our notation,
\[
    \nabla_\theta \tilde{F}_{\Omega, Y}(\theta^*) 
           = {\rm vec} \left( \frac{\partial F_{\Omega, Y}(M^*)}{\partial M_{\ell k}}  \right)\,.
\]
We then have
\begin{equation} \label{geq3.10}
  \langle \nabla_{\theta} \tilde{F}_{\Omega, Y}(\theta^*),w \rangle 
    = \langle \nabla_M F_{\Omega,Y}(M^*), M-M^* \rangle
\end{equation}
where $\langle A,B \rangle := {\rm tr}(A^\top B )$.
Let $Z\equiv \nabla_M F_{\Omega,Y}(M^*)$. Therefore,
\begin{align*} 
 Z_{ij} 
    = \left( - \frac{\dot{f}(M_{ij})}{f(M_{ij})} \ind_{(Y_{ij} = 1)} 
      + \frac{\dot{f}(M_{ij})}{1-f(M_{ij})} \ind_{(Y_{ij} = -1)} \right)  
               \ind_{((i,j) \in \Omega)}\,.
\end{align*}
Using (\ref{e1.1a}) and (\ref{geq108}), we have
\begin{equation} \label{geq3.15}
  \mathbb{E} [Z_{ij}]=0, \; |Z_{ij}| \le L_{\alpha} \; \implies \; \mathbb{E} [Z_{ij}^2] \le  L_{\alpha}^2   \, .
\end{equation}

We need the following result from \cite{Chatterjee12} concerning spectral norms of random matrices for Lemma \ref{lemma:grad}.
\begin{lemma} \label{lemma:Chatterjee} \cite[Theorem 8.4]{Chatterjee12}
Take any two numbers $m$ and $n$ such that $1 \le n\le m$. Suppose that $A= [A_{ij}]_{1\le i \le m, 1\le j \le n}$ is a matrix whose entries 
are independent random variables that satisfy, for some $\sigma^2 \in [0,1]$,
\[
    \mathbb{E} [ A_{ij} ] = 0, \; \mathbb{E}[A_{ij}^2 ] \le \sigma^2, \mbox{ and } |A_{ij}| \le 1 \; a.s.
\]
Suppose that $\sigma^2 \ge m^{-1+\varepsilon}$ for some $\varepsilon >0$. Then
\[
    P \left( \|A\|_2 \ge 2.01 \sigma \sqrt{m} \right) \le C_1(\varepsilon) e^{-C_2 \sigma^2 m} ,
\]
where $C_1(\varepsilon)$ is a constant that depends only on $\varepsilon$ and $C_2$ is a positive universal constant.
The same result is true when $m=n$ and $A$ is symmetric or skew-symmetric, with independent entries on and above the diagonal, all other assumptions remaining the same. Lastly, all results remain true if the assumption $\sigma^2 \ge m^{-1+\varepsilon}$ is changed to $\sigma^2 \ge m^{-1} ( \log(m) )^{6+\varepsilon}$.
\end{lemma}
\begin{lemma} \label{lemma:grad}
Let $w \equiv  {\rm vec}(M-M^*)= \theta - \theta^*$, and $M, M^* \in {\mathcal C}$. 
Then with probability at least $1 - C_1(\varepsilon) \exp (-C_2 m)$,
we have
\begin{align*}
   \left| \langle \nabla_\theta \tilde{F}_{\Omega, Y}(\theta^*),w \rangle \right| & 
             \le        2.01 L_{\alpha} \sqrt{2 r m} \|M-M^*\|_F\,,
\end{align*}
where 
$\varepsilon \in (0,1)$, $C_1(\varepsilon)$ is a constant that depends only on $\varepsilon$ and $C_2$ is a positive universal constant.
\end{lemma}
\begin{proof} 
Using (\ref{geq3.10}), we have
\begin{align} 
 |\langle \nabla_{\theta} & \tilde{F}_{\Omega, Y}(\theta^*),w \rangle| 
  = | \langle \nabla_M F_{\Omega,Y}(M^*), M-M^* \rangle | \nonumber \\
  &  \le \|\nabla_M F_{\Omega,Y}(M^*) \|_2 \|M^*-M\|_* .  \label{geq3.20}
\end{align}
Let $\tilde{Z} \equiv  L_{ \alpha}^{-1} \nabla_M F_{\Omega,Y}(M^*)$. Then we have
$\mathbb{E}[ \tilde{Z}_{ij} ] = 0$, $|\tilde{Z}_{ij}| \le 1$ and 
$\mathbb{E}[ \tilde{Z}_{ij}^2 ] \le 1$. 
Applying Lemma \ref{lemma:Chatterjee} to $\tilde{Z}$ with $\sigma = 1$, we obtain $\|\tilde{Z} \|_2 \le 2.01 \sqrt{m}$ with probability at least $1-C_1(\varepsilon) \exp (-C_2 m)$ for some positive constants $C_1(\varepsilon)$ and $C_2$. W note that for any matrix $A$ of rank $r$, $\|A\|_* \le \sqrt{r} \|A\|_F$ with $\|A\|_*$ denoting the nuclear norm. Hence
$\|M^*-M\|_* \le \sqrt{2r} \|M^*-M\|_F$, yielding the desired result.
\end{proof}

\begin{lemma} \label{lemma:hess}
Let $w = {\rm vec}(M-M^*) = \theta - \theta^*$ and $M, M^* \in {\mathcal C}$. 
Then for any $\tilde{\theta} = \theta^* + \gamma (\theta - \theta^*)$ and any $\gamma \in [0,1]$, we have
\begin{align*}
    \langle w,  \left[ \nabla_{\theta \theta}^2 \tilde{F}_{\Omega, Y}(\tilde{\theta}) \right] w \rangle    
    & \ge  \gamma_{ \alpha} \left\| \left( M-M^* \right)_\Omega \right\|_F^2  .
\end{align*}
\end{lemma}
\begin{proof}
Using (\ref{geq104}), (\ref{geq3.2}) and (\ref{geq3.3}), we have
\begin{align} 
\langle w, & \left[ \nabla_{\theta \theta}^2 \tilde{F}_{\Omega, Y}(\tilde{\theta}) \right] w \rangle \nonumber \\
   & = \sum_{(i,j) \in \Omega} \left( \frac{\partial^2 F_{\Omega, Y}(\tilde{M})}{\partial M_{ij}^2} \right) (M_{ij}-M^*_{ij})^2 \nonumber \\
   & \ge \gamma_{\alpha} \sum_{(i,j) \in \Omega}  (M_{ij}-M^*_{ij})^2
      = \gamma_{\alpha} \left\| \left( M-M^* \right)_\Omega \right\|_F^2\,,  \label{geq3.41}
\end{align}
which completes the proof.
\end{proof}

We need a result similar to \cite[Theorem 4.1]{Bhojanapalli:2014kx} regarding closeness of a fixed matrix to its sampled version, which is proved therein for square matrices $M^*$ under an incoherence assumption on $M^*$.
In Lemma \ref{lemma:partial} we prove a similar result for rectangular $Z$ with bounded $\| Z \|_\infty$. Define 
\[\|Z\|_{\max} \equiv \inf\{\max(\|U\|_{2,\infty}^2,\|V\|_{2,\infty}^2):\, Z = UV^\sT\}\,,\]
where for a matrix $A$, $\|A\|_{2,\infty}$ denotes the largest $\ell_2$ norm of the rows in $A$ , i.e, $\|A\|_{2,\infty}\equiv \max_{i}\|U_{i,\cdot}\|_2$.

For $Z \in \reals^{m \times n}$, $m \ge n$, and define the operator
${\cal R}_\Omega$ as 
$$Z_\Omega \equiv {\cal R}_\Omega(Z)  = \begin{cases}
Z_{ij} &\mbox{ if } (i,j) \in \Omega,\\
0&\mbox{ otherwise. }
\end{cases}
$$

\begin{lemma} \label{lemma:partial} 
Let $G\backslash\Omega$ satisfy assumptions (A1) and (A2) in Section~\ref{sec:alg}.  
Let $Z \in \mathbb{R}^{m \times n}$ with rank$(Z) \le r$. Then we have
\begin{align}
  & \left\| \left( \frac{ \sqrt{mn} }{ \sigma_1 (G) } R_\Omega -I \right) (Z) \right\|_2 \leq  \frac{ \sqrt{mn} \sigma_2 (G) } 
                { \sigma_1 (G) } \| Z \|_{\max} \label{partiala} \\
      &  \leq  \frac{ \sqrt{rmn} \sigma_2 (G) } 
                { \sigma_1 (G) } \| Z \|_\infty
                  \leq C m \sqrt{ \frac{n r} { |\Omega|} } \| Z \|_\infty  .  \label{partialb}
\end{align}
\end{lemma}
\begin{proof} 
By definition of $\|Z\|_{\max}$, there exist $U \in \reals^{m \times k}$ and $V \in \reals^{n \times k}$ for some $1 \le k \le \min(m,n)$ such that $Z = U V^\top$,  $\|U\|_{2,\infty}^2 \le \|Z\|_{\max}$ and $\|V\|_{2,\infty}^2 \le \|Z\|_{\max}$.
Since $\text{rank}(Z) \le r$, we have $k \le r$, but this fact is not needed in our proof.
By the variational definition of operator norm,
\[
  \| \frac{\sqrt{mn}}{\sigma_1(G)} R_{\Omega}(Z) - Z \|_2 
\]
\[
   = \max_{x,y:\, \|x\|_2 =1 = \|y\|_2} 
       y^\top \left( \frac{\sqrt{mn}}{\sigma_1(G)} R_{\Omega}(Z) - Z \right) x .
\] 
We also have $R_{\Omega}(Z) = Z \circ G$ where $\circ$ denotes the Hadamard (elementwise) product. Letting $U_{\cdot,\ell}$ and $V_{\cdot,\ell}$ respectively denote the $\ell$-th column of $U$ and $V$, we write
\[Z = \sum_{\ell=1}^k U_{\cdot,\ell} V_{\cdot,\ell}^\top\,,\]
We therefore have
\begin{align*}
 & y^\top  \left( \frac{\sqrt{mn}}{\sigma_1(G)} {\cal R}_{\Omega}(Z) - Z \right) x  \\
   & = \sum_{i=1}^k \left( \frac{\sqrt{mn}}{\sigma_1(G)} (y \circ U_{\cdot,\ell} )^\top G (x \circ V_{\cdot,\ell}) - (y^\top U_{\cdot,\ell}) (x^\top V_{\cdot,\ell}) \right)\,.
\end{align*}
Normalize $\mathbf{1}_{m}$ to unit norm as $\tilde{\mathbf{1}}_{m} = \mathbf{1}_{m}/\sqrt{m}$, and similarly for $\tilde{\mathbf{1}}_{n}$.
Let $y \circ U_{\cdot,\ell} = \alpha_\ell \tilde{\mathbf{1}}_{m} + \beta_\ell \tilde{\mathbf{1}}_{m \perp}^\ell$ where $\tilde{\mathbf{1}}_{m \perp}^\ell$ is a unit norm vector orthogonal to $\tilde{\mathbf{1}}_{m}$. Then $\alpha_\ell = \tilde{\mathbf{1}}_{m}^\top (y \circ U_{\cdot,\ell}) = y^\top U_{\cdot,\ell}/\sqrt{m}$. Hence
\begin{align}
  y^\top & \left( \frac{\sqrt{mn}}{\sigma_1(G)} {\cal R}_\Omega (Z) - Z \right) x \nonumber \\
    & = \sum_{\ell=1}^k \Big( \frac{\sqrt{mn}}{\sigma_1(G)} \Big[ \frac{1}{\sqrt{m}} y^\top U_{\cdot,\ell} 
         \tilde{\mathbf{1}}_{m}^\top G (x \circ V_{\cdot,\ell})  \nonumber \\
    & \quad \quad   +\beta_\ell \tilde{\mathbf{1}}_{m \perp}^{\ell \top}  G (x \circ V_{\cdot,\ell}) \Big] - (y^\top U_{\cdot,\ell}) (x^\top V_{\cdot,\ell}) \Big) \nonumber \\
   & = \sum_{\ell=1}^k \left( \frac{\sqrt{mn}}{\sigma_1(G)} \beta_\ell \tilde{\mathbf{1}}_{m \perp}^{\ell \top}  G (x \circ V_{\cdot,\ell})  \right),  \label{eqA_10}
\end{align}
where we used the facts that $\tilde{\mathbf{1}}_{m}^\top G = \sigma_1(G) \tilde{\mathbf{1}}_{n}^\top$ and $\tilde{\mathbf{1}}_{n}^\top (x \circ V_{\cdot,\ell}) = x^\top V_{\cdot,\ell}/\sqrt{n}$. Since $\tilde{\mathbf{1}}_{m}$ is the top left singular vector of $G$, we have
\[
   | \tilde{\mathbf{1}}_{m \perp}^{\ell\top}  G z | \le \sigma_2(G) \|z\|_2 \mbox{ for any } z \in \reals^{n} .
\]
Using the above inequality in (\ref{eqA_10}) we obtain
\begin{align}
  y^\top & \left( \frac{\sqrt{mn}}{\sigma_1(G)} R_\Omega (Z) - Z \right) x \nonumber \\
    & \le \frac{\sqrt{mn}}{\sigma_1(G)} \sigma_2(G) \sum_{\ell=1}^k |\beta_\ell| \|x \circ V_{\cdot,\ell}\|_2  \nonumber \\
    & \le \frac{\sqrt{mn}}{\sigma_1(G)} \sigma_2(G) \sqrt{\sum_{\ell=1}^k \beta_\ell^2} \sqrt{\sum_{\ell=1}^k \|x \circ V_{\cdot,\ell}\|_2^2 } . \label{eqA_15}
\end{align}
We have $\beta_\ell = \tilde{\mathbf{1}}_{m\perp}^{\ell \top} (y \circ U_{\cdot,\ell})$. Hence, $|\beta_\ell| \le \|y \circ U_{\cdot,\ell} \|_2$. Therefore,
\begin{align}
    \sum_{\ell=1}^k \beta_\ell^2 & \le \sum_{\ell=1}^k \|y \circ U_{\cdot,\ell} \|_2^2 = \sum_{i = 1}^{m}  \sum_{\ell = 1}^k y_i^2 U_{i,\ell}^2 \nonumber \\
      & = \sum_{i = 1}^{m} y_i^2   \|U_{i,\cdot}\|_2^2 \le \|U\|_{2,\infty}^2 \sum_{i = 1}^{m} y_i^2 \le \|Z\|_{\max}   \label{eqA_18}
\end{align}
where we used $\sum_{i = 1}^{m} y_i^2 = 1$. Similarly, we have
\begin{align}
  \sum_{\ell=1}^k & \|x \circ V_{\cdot,\ell}\|_2^2  = \sum_{j = 1}^{n}  \sum_{\ell = 1}^k x_j^2 V_{j,\ell}^2 \nonumber \\
      & = \sum_{j = 1}^{n} x_j^2   \|V_{j,\cdot}\|_2^2 \le \|V\|_{2,\infty}^2 \sum_{j = 1}^{n} x_j^2 \le \|Z\|_{\max}  \, .  \label{eqA_20}
\end{align}
It then follows from (\ref{eqA_15})-(\ref{eqA_20}) that
\begin{align}
  y^\top  \left( \frac{\sqrt{mn}}{\sigma_1(G)} R_\Omega (Z) - Z \right) x 
    & \le \frac{ \sqrt{mn} \sigma_2 (G) } { \sigma_1 (G) } \| Z \|_{\max} \nonumber \\
\end{align}
This establishes (\ref{partiala}). Now use $\|Z\|_{\max} \le \sqrt{r} \|Z\|_\infty$ \cite{Cai13} and $| \Omega | = m d$ to establish (\ref{partialb}).
\end{proof}

\begin{lemma} \label{lemma:hpartial}
Let $M, M^* \in {\mathcal C}$. Then we have
\begin{align*}
    \left\| \left( M-M^* \right)_\Omega \right\|_F  \ge \frac{\sigma_1(G)}{\sqrt{2rmn}} \left\|  M-M^*  \right\|_F
         - 2 \alpha \sqrt{r} \sigma_2(G) .
\end{align*}
\end{lemma}
\begin{proof}
Let $Z\equiv M-M^*$, $a = \sqrt{mn}/\sigma_1(G)$, and $b = (\sigma_2(G) / \sigma_1(G)) \sqrt{  r mn }$. Then by Lemma \ref{lemma:partial} and the fact that 
rank$(Z) \le {\rm rank}(M) + {\rm rank}(M^*) \le 2r$, we have 
\begin{align}  
  \left| a \| Z_\Omega \|_2  -  \|  Z \|_2 \right| & \le \| a Z_\Omega - Z \|_2  
      \le b \|Z\|_\infty .  \label{eq7.1}
\end{align}
Using $\|Z\|_\infty = \|M-M^*\|_\infty \le \|M\|_\infty + \|M^*\|_\infty \le 2 \alpha$, (\ref{eq7.1}) can be expressed as
$\|  Z \|_2 \le a \| Z_\Omega \|_2 + 2 \alpha b$.
Since $\|A\|_2 \le \|A\|_F$ $\forall A$, we then have
$\|  Z \|_2 \le a \| Z_\Omega \|_F + 2 \alpha b$.
Since $\|A\|_F \le \sqrt{{\rm rank}(A)} \|A\|_2$ $\forall A$, we  have
$\|  Z \|_F  \le \sqrt{2r} \|Z\|_2  \le \sqrt{2r} a \| Z_\Omega \|_F + 2\sqrt{2r}  \alpha b$, leading to
the desired result.
\end{proof}

{\em Proof of Theorem \ref{thm:main_gd}:} 
Consider $\tilde{F}_{\Omega, Y}(\theta) = {F}_{\Omega, Y}(M)$. The objective function $F_{\Omega, Y}(M)$ is continuous in $M$ and the set
${\mathcal C}$ is compact, therefore, $F_{\Omega, Y}(M)$ achieves a minimum in ${\mathcal C}$. Now suppose that $\widehat{M} \in {\mathcal C}$ minimizes ${F}_{\Omega, Y}(M)$. Then ${F}_{\Omega, Y}(\widehat{M}) \le {F}_{\Omega, Y}(M)$ $\forall M \in {\mathcal C}$, including $ M = M^*$.
Define
\begin{equation}
   c_g  = 2.01 L_{ \alpha} \sqrt{2 r m} ,  \quad \quad
   c_h  = \frac{\sigma_1^2(G) \gamma_{ \alpha} }{16  r mn}\,.
\end{equation}
Using (\ref{geq3.0}) and Lemmas \ref{lemma:grad} and \ref{lemma:hess}, we have w.h.p.\ (specified in Lemma \ref{lemma:grad})
\begin{align*}  
    &{F}_{\Omega, Y}(M)  \\
    &\ge {F}_{\Omega, Y}(M^*) - c_g \|M-M^*\|_F  + \frac{\gamma_{ \alpha}}{2 } \|( M-M^* )_\Omega \|_F^2.
\end{align*}
Since $\widehat{M}$ minimizes ${F}_{\Omega, Y}(M)$, we have 
\begin{align}  
   0 & \ge  {F}_{\Omega, Y}(\widehat{M}) - {F}_{\Omega, Y}(M^*) \nonumber \\
     &  \ge - c_g \|\widehat{M}-M^*\|_F + \frac{\gamma_{ \alpha}}{2 } \| ( \widehat{M}-M^* )_\Omega \|_F^2 . \label{geq4.03} 
\end{align}
Set $\eta = 2 \alpha r (\sigma_2(G) / \sigma_1(G)) \sqrt{2mn}$ and $\eta_0 = \sigma_1(G)/ \sqrt{2rmn}$. Then Lemma \ref{lemma:hpartial} implies $\| ( M-M^*)_\Omega \|_F \ge \eta_0 \left[ \|M-M^*\|_F - \eta \right]$. Now consider two cases: (i) $\|\widehat{M}-M^*\|_F < 2\eta$, (ii) $\|\widehat{M}-M^*\|_F \ge 2\eta$. In case (i), we clearly have an obvious upperbound on $\|\widehat{M}-M^*\|_F$. Turning to case (ii), we have
\begin{align} 
  \|\widehat{M}-M^*\|_F - \eta & \ge \|\widehat{M}-M^*\|_F - \frac{1}{2} \|\widehat{M}-M^*\|_F \nonumber \\
  & = \frac{1}{2} \|\widehat{M}-M^*\|_F . \label{geq4.05}
\end{align}
Using (\ref{geq4.03}), (\ref{geq4.05}) and Lemma \ref{lemma:hpartial} with $M = \widehat{M}$, we have
\begin{align}
  0 & \ge  {F}_{\Omega, Y}(\widehat{M}) - {F}_{\Omega, Y}(M) \nonumber \\
        & \ge  - c_g \|\widehat{M}-M\|_F + {c_h}  \|\widehat{M}-M\|_F^2 \nonumber \\
         & = \|\widehat{M}-M\|_F \left[ - c_g + {c_h}  \|\widehat{M}-M\|_F \right] .
                        \label{geq4.0}
\end{align}
In order for (\ref{geq4.0}) to be true, we must have 
$\|\widehat{M}-M^*\|_F \le c_g / c_h$
otherwise the right-side of (\ref{geq4.0}) is positive violating (\ref{geq4.0}). Combining the two cases, we obtain
\begin{align} 
   \|\widehat{M}- & M^*\|_F  \le \max \left( 2 \eta , \frac{c_g}{c_h} \right) \nonumber \\
     &  = \max \left(4 \alpha r \sqrt{2mn} \frac{\sigma_2(G)}{\sigma_1(G)} , \frac{32.16 \sqrt{2} L_{ \alpha} (rm)^{1.5}
                n }{\gamma_{ \alpha} \sigma_1^2(G) } \right) \label{geq4.08}
\end{align}
This is the bound stated in (\ref{thm_eq1}) of the theorem after division by $\sqrt{mn}$.
The high probability stated in the theorem follows from Lemma \ref{lemma:grad} after setting $\varepsilon =0.5$. Finally, we use 
$\sigma_2(G)/\sigma_1(G) \le C/\sqrt{d} = C \sqrt{m} /\sqrt{|\Omega|}$ and 
$1/\sigma_1^2(G) \le 1/d^2 =  m^2 /|\Omega|^2$ to derive (\ref{thm_eq2}). 
$\hfill \;\; \blacksquare$

{
\bibliographystyle{IEEEtran} 
\bibliography{IEEEabrv,My-bibliography}
}

\end{document}